\documentclass{tlp}

\usepackage{tkz-graph}
\usepackage{pgfplots}
\usepackage{bbm} 
\usepackage{amssymb}
\usepackage{hyperref}
\usepackage{algorithm}
\usepackage{algpseudocode}

\usepackage{subcaption}

\definecolor{red_plot}{HTML}{F30522}
\definecolor{first_plot}{HTML}{F30522}
\definecolor{original_plot}{HTML}{F30522}

\definecolor{orange_plot}{HTML}{FA5F22}

\definecolor{yellow_plot}{HTML}{DED712}
\definecolor{fourth_plot}{HTML}{DED712}

\definecolor{pink_plot}{HTML}{FFD0B4}
\definecolor{light_blue_plot}{HTML}{D1E5F0}
\definecolor{blue_plot}{HTML}{67A9CF}

\definecolor{dark_blue_plot}{HTML}{2166AC}
\definecolor{third_plot}{HTML}{2166AC}
\definecolor{tabled_plot}{HTML}{2166AC}

\definecolor{green_plot}{HTML}{20D525}
\definecolor{second_plot}{HTML}{20D525}

\newcommand{\resizeGraphFactor}{0.80}

\newcommand{\aspmcr}{$\mathrm{aspmc}^r$}

\newcommand{\impl}{\ {:\!-}\  } 

\DeclareMathOperator{\IFP}{\mathit{IFP}}
\DeclareMathOperator{\WFM}{\mathit{WFM}}
\DeclareMathOperator{\OT}{\mathit{OT}}
\DeclareMathOperator{\OF}{\mathit{OF}}
\DeclareMathOperator{\lfp}{\mathit{lfp}}
\DeclareMathOperator{\gfp}{\mathit{gfp}}

\DeclareMathOperator{\lowerprob}{\underline{P}}
\DeclareMathOperator{\upperprob}{\overline{P}}
\newtheorem{example}{Example}
\newtheorem{definition}{Definition}
\newtheorem{theorem}{Theorem}
\newtheorem{lemma}{Lemma}

\usepackage{listings}
\begin{document}

\lefttitle{Cambridge Author}

\jnlPage{1}{16}
\jnlDoiYr{2024}
\doival{10.1017/xxxxx}

\title[Fast Inference for Probabilistic Answer Set Programs]{Fast Inference for Probabilistic Answer Set Programs via the Residual Program}

\begin{authgrp}
\author{\gn{Damiano} \sn{Azzolini} }
\affiliation{Department of Environmental and Prevention Sciences -- University of Ferrara, Ferrara, Italy \\
\email{damiano.azzolini@unife.it}}
\author{\gn{Fabrizio} \sn{Riguzzi} }
\affiliation{Department of Mathematics and Computer Science -- University of Ferrara, Ferrara, Italy\\
\email{fabrizio.riguzzi@unife.it}}
\end{authgrp}

\history{\sub{xx xx xxxx;} \rev{xx xx xxxx;} \acc{xx xx xxxx}}

\maketitle

\begin{abstract}
When we want to compute the probability of a query from a Probabilistic Answer Set Program, some parts of a program may not influence the probability of a query, but they impact on the size of the grounding.
Identifying and removing them is crucial to speed up the computation.
Algorithms for SLG resolution offer the possibility of returning the residual program which  can be used for computing answer sets for normal programs that do have a total
well-founded model.  The residual program does not contain the parts of the program that do not influence the probability.
In this paper, we propose to exploit the residual program for performing inference.
Empirical results on graph datasets show that the approach leads to significantly faster inference.
The paper has been accepted at the ICLP2024 conference and under consideration in Theory and Practice of Logic Programming (TPLP).
\end{abstract}


\maketitle

\section{Introduction}
\label{sec:introduction}
Statistical Relational Artificial Intelligence~\citep{raedt2016statistical} is a subfield of Artificial Intelligence aiming at representing uncertain domains with interpretable languages.
One of these languages is Probabilistic Answer Set Programming (PASP) under the credal semantics, i.e., Answer Set Programming (ASP, and we use the same acronym to denote answer set programs) extended with probabilistic facts.
Inference in PASP 
often requires grounding the whole program, due to the model driven ASP solving approach.
However, other formalisms based on query driven languages, such as PITA~\citep{riguzzi2011pita} and ProbLog2~\citep{dries2015problog2}, only ground the relevant part of the program.
For a specific class of ASP, namely normal programs without odd loops over negation, we propose to extract the relevant program using SLG resolution~\citep{DBLP:journals/jacm/ChenW96}, that offers the possibility of returning the \emph{residual program}, which can then be used to compute the answer sets of the program.
At a high level, the process is the following: first, we convert a PASP into a Prolog program that is interpreted under the Well-founded semantics.
Then, we leverage SLG resolution via tabling to compute the residual program for a given query.
Lastly, we convert the residual program into a PASP, often smaller than the original PASP, and call a solver to compute the probability of the query.  
In this way we reduce the size of the program that should be grounded, consistently speeding up the execution time, as demonstrated by different experiments on graph datasets.

The paper is structured as follows: Section~\ref{sec:background} provides some background knowledge, Section~\ref{sec:residual_program_extraction} introduces our solution to extract the residual program, which is tested in Section~\ref{sec:experiments}.
Section~\ref{sec:related} discusses related works and Section~\ref{sec:conclusions} concludes the paper.

\section{Background}
\label{sec:background}
In this paper, we consider normal logic programs, i.e., programs composed of normal rules of the form $r=h \impl b_1, \dots, b_m, not\ c_1,\ldots, not\ c_m$, where $h$, $b_i$  for $i=1,\ldots,m $ and $c_j$ for $j=1,\ldots,n$ are atoms.
Given a rule $r$, we call $H(r)=h$, $B^+(r)=\{b_1,\ldots,b_m\}$, $B^-(r)=\{c_1,\ldots,c_n\}$, $B(r)=\{b_1, \dots, b_m, not\ c_1,\ldots, not\ c_m\}$ the head, positive body, negative body and body of $r$.
A rule with an empty body is called fact.
We indicate the Herbrand base of a program $P$ with $B_P$ and its grounding with $ground(P)$.
We use the standard notation $name/arity$ to denote a predicate with name $name$ and arity, i.e., number of arguments, $arity$.
The \emph{call graph} of a program $P$ is a directed graph with one node for each predicate in the program.
There is an edge between a predicate $p/n$ and a predicate $q/m$ if $p/n$ is the predicate of the head atom of a rule and $q/m$ is the predicate of a literal in the body of that rule.
The edge is labeled as positive ($+$) or negative ($-$) depending on whether the literal is positive or negative in the body of the considered rule.
A program $P$ includes \emph{Odd Loops Over Negation} (OLON) if its call graph contains a cycle with an odd number of negations.
Figure~\ref{fig:call_graph} shows examples of programs with and without OLON, together with their call graphs.
The \emph{dependency graph} of a program $P$ is a directed graph with one node for each atom in the Herbrand base of the program.
There is an edge between an atom $a$ and an atom $b$ if $a$ is the head of a rule in the grounding of $P$ and $b$ is the atom of a literal in the body of that rule.
A semantics is \emph{relevant}~\citep{10.5555/2383266.2383269} if the truth value of an atom $a$ depends only from the truth value of the atoms of the \emph{relevant sub-graph} of the dependency graph, i.e., the sub-graph  that contains the nodes that are reachable from $a$.

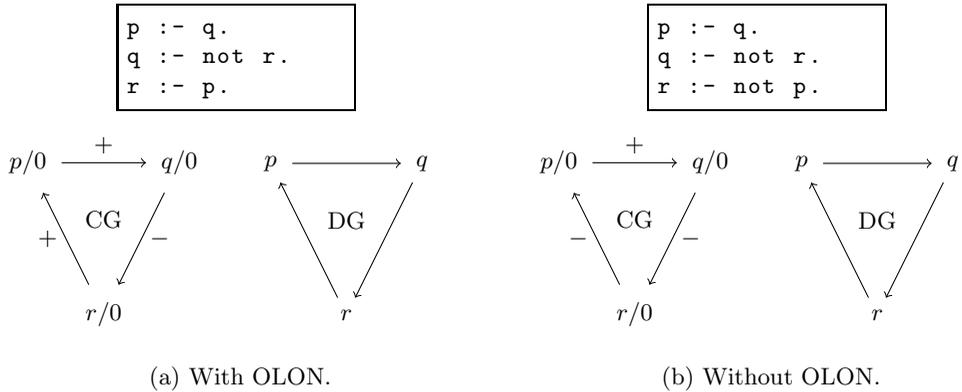
\begin{figure}[t]
\begin{subfigure}{0.48\textwidth}
  \begin{minipage}{0.25\textwidth}
    ~
  \end{minipage}
\begin{minipage}[c]{0.45\textwidth}
  \begin{lstlisting}
p :- q.
q :- not r.
r :- p.
  \end{lstlisting}
  \end{minipage}
  \\
  \begin{minipage}[c]{0.48\textwidth}
    \begin{tikzpicture}
      \node[circle, inner sep=3pt] (p) at (0,0) {$p/0$};
      \node[circle, inner sep=3pt] (q) at (2,0) {$q/0$};
      \node[circle, inner sep=3pt] (r) at (1,-2) {$r/0$};
      \node[circle, inner sep=3pt] (CG) at (1,-0.75) {CG};
    
    
      \draw[->] (p) -- node[above] {$+$} (q);
      \draw[->] (q) -- node[right] {$-$} (r);
      \draw[->] (r) -- node[left] {$+$} (p);
    \end{tikzpicture}
  \end{minipage}
  \hfill
  \begin{minipage}[c]{0.48\textwidth}
    \begin{tikzpicture}
      \node[circle, inner sep=3pt] (p) at (0,0) {$p$};
      \node[circle, inner sep=3pt] (q) at (2,0) {$q$};
      \node[circle, inner sep=3pt] (r) at (1,-2) {$r$};
      \node[circle, inner sep=3pt] (DG) at (1,-0.75) {DG};

    
      \draw[->] (p) -- node[above] {} (q);
      \draw[->] (q) -- node[right] {} (r);
      \draw[->] (r) -- node[left] {} (p);
    \end{tikzpicture}
  \end{minipage}
  \caption{With OLON.}
  \label{subfig:olon}
\end{subfigure}
  \hfill
  \begin{subfigure}{0.48\textwidth}
    \begin{minipage}{0.25\textwidth}
      ~
    \end{minipage}
  \begin{minipage}[c]{0.45\textwidth}
    \begin{lstlisting}
p :- q.
q :- not r.
r :- not p.
    \end{lstlisting}
    \end{minipage}
    \\
  \begin{minipage}[c]{0.48\textwidth}
    \begin{tikzpicture}
      \node[circle, inner sep=3pt] (p) at (0,0) {$p/0$};
      \node[circle, inner sep=3pt] (q) at (2,0) {$q/0$};
      \node[circle, inner sep=3pt] (r) at (1,-2) {$r/0$};
      \node[circle, inner sep=3pt] (CG) at (1,-0.75) {CG};
    
    
      \draw[->] (p) -- node[above] {$+$} (q);
      \draw[->] (q) -- node[right] {$-$} (r);
      \draw[->] (r) -- node[left] {$-$} (p);
    \end{tikzpicture}
  \end{minipage}
  \hfill
  \begin{minipage}[c]{0.48\textwidth}
    \begin{tikzpicture}
      \node[circle, inner sep=3pt] (p) at (0,0) {$p$};
      \node[circle, inner sep=3pt] (q) at (2,0) {$q$};
      \node[circle, inner sep=3pt] (r) at (1,-2) {$r$};
      \node[circle, inner sep=3pt] (DG) at (1,-0.75) {DG};

    
      \draw[->] (p) -- node[above] {} (q);
      \draw[->] (q) -- node[right] {} (r);
      \draw[->] (r) -- node[left] {} (p);
    \end{tikzpicture}
    \end{minipage}
    \caption{Without OLON.}
    \label{subfig:no_olon}
  \end{subfigure}
\caption{Programs, call graphs (CD), and dependency graphs (DG) with (Figure~\ref{subfig:olon}) and without (Figure~\ref{subfig:no_olon}) OLON.
The dependency graph of both programs is the same.
They only differ in the call graph: for the left program, the call graph contains an edge labeled with $-$ (negative) while for the right program the same edge is labeled with $+$ (positive).}
\label{fig:call_graph}
\end{figure}

\subsection{Stable Model Semantics}
\label{subsec:stable_model_semantics}
The Stable Model Semantics (SMS)~\citep{gelfond1988stable} associates zero or more stable models to logic programs.
An interpretation is a subset of $B_P$.
The \emph{reduct} of a ground program $P$ w.r.t. an interpretation $I$, $P^I$, also known as Gelfond-Lifschitz reduct, is the set of rules in the grounding of $P$ that have their body true in $I$, that is
$P^I=\{r\in ground(P) \mid B^+(r)\subseteq I, B^-(r)\cap I=\emptyset\}$.
A \emph{stable model} or \emph{answer set} (AS) of  a program $P$ is an interpretation $I$ such that $I$ is a minimal model under set inclusion of  $P^I$.
With $AS(P)$ we denote the set of answer sets of a program $P$.
We also consider \emph{projected answer sets}~\citep{gebser2009projective} on a set of ground atoms $V$, defined as $AS_V(P) = \{A \cap V \mid A \in AS(P)\}$.
Answer Set Programming (ASP)~\citep{brewka2011asp} considers programs under the SMS.

\begin{example}
\label{ex:asp_example}
The following program is composed by (in order of appearance) three facts and four normal rules.
\begin{lstlisting}
e(a,b). e(a,c). e(b,d).
edge(A,B):- e(A,B), not nedge(A,B).
nedge(A,B):- e(A,B), not edge(A,B).
path(A,B):- edge(A,B).
path(A,B):- edge(A,C), edge(C,B).
\end{lstlisting}
It has the following 8 answer sets (we only report the $path/2$ atoms, for brevity):
$\{\}$,
$\{path(a,c)\}$,
$\{path(a,b)\}$,
$\{path(b,d)\}$,
$\{path(a,c), path(b,d)\}$,
$\{path(a,b), path(a,c)\}$,
$\{path(a,d), path(a,b), path(b,d)\}$, and
$\{path(a,d), path(a,b), path(a,c), path(b,d)\}$.
These answer sets are exactly the ones obtained by projecting the original answer sets on the $path/2$ atoms. 
\end{example}

ASP has been extended to consider various extensions of normal logic programs: disjunction, constraints, explicit negation, and aggregates. 
The SMS for normal programs is not relevant, as is shown by the following example.
\begin{example}
\label{ex:asp_not_relevant}
Consider the query $q$ from the program \{$q \impl a.$, $a.$\}.
This program has a single AS $I=\{q,a\}$ so $q$ is true in all AS (we say that $q$ is \emph{skeptically true}).
However, if we add the constraint $c \impl not \ c$,
the program has no AS, so $q$ is no more skeptically true, even if the dependency graph of $q$ includes only the node $q$ and $a$.
\end{example}
In this paper, we restrict to normal programs without OLON.
These programs always have at least one AS and the SMS in this case is relevant~\citep{DBLP:journals/tplp/MarpleG14}.

\subsection{Well-founded Semantics}
\label{subsec:wfs}
The Well-founded Semantics (WFS)~\citep{well-founded} assigns a three valued model to a program.
A three valued interpretation $I$ is a pair $I=\langle I_T ; I_F \rangle$ where both $I_T$ and $I_F$ are disjoint subsets of $B_P$ and represent the sets of true and false atoms, respectively.
Given a three valued interpretation $I = \langle I_T ; I_F \rangle$ for a program $P$, an atom $a$ is
i) true in $I$ if $a \in I_T$ and
ii) false in $I$ if $a \in I_F$
while an atom $not \ a$ is 
i) true in $I$ if $a \in I_F$ and
ii) false in $I$ if $a \in I_T$.
If $a$ does not belong neither to $I_T$ nor $I_F$ it is  \textit{undefined}.
Furthermore, we define the functions $t(I)$, $f(I)$, and $u(I)$ returning the true, false, and undefined atoms, respectively.
Lastly, we can define a partial order on three valued interpretations as $\langle I_T ; I_F \rangle \leq \langle J_T ; J_F \rangle$ if $I_T \subseteq J_T$ and $I_F \subseteq J_F$.

We recall here the  iterated fixpoint definition of the WFS from~\cite{Przymusinski89}. Consider two sets of ground atoms, $T$ and $F$, a normal logic program $P$, and a three valued interpretation $I$.
We define the following two operators:
\begin{itemize}
  \item $\OT_{I}^{P}(T) = \{a \mid a$ is not true in $I$ and there exist a clause $h \leftarrow l_1,\dots,l_m$ of $P$ such that $a = h\theta$ for a grounding substitution $\theta$ of the clause and $\forall i \in \{1,\dots,m\}$, $l_i\theta$ is true in $I$ or $l_i\theta \in T\}$ and
  \item $\OF_{I}^{P}(F) = \{a \mid a$ is not false in $I$ and for every clause $h \leftarrow l_1,\dots,l_m$ and every grounding substitution $\theta$ of the clause of $P$ such that $a = h\theta$ there exist an $i \in \{1,\dots,m\}$ such that $l_i\theta$ is false in $I$ or $l_i\theta \in F\}$.
\end{itemize}
In other words, $\OT_{I}^{P}(T)$ is the set of atoms that can be derived from $P$ knowing $I$ and $T$ while $\OF_{I}^{P}(F)$ is the set of atoms that can be shown false in $P$ knowing $I$ and $F$.
\cite{Przymusinski89} proved that both operators are monotonic and so they have a least and greatest fixpoint ($\lfp$ and $\gfp$, respectively).
Furthermore, the iterated fixpoint operator $\IFP^P(I) = I \cup \langle \lfp(\OT_{I}^{P}), \gfp(\OF_{I}^{P}) \rangle$ has also been proved monotonic by~\cite{Przymusinski89}.
The Well-Founded model (WFM) of a normal program $P$ is the least fixpoint of $\IFP^P$, i.e., $\WFM(P)=\lfp(\IFP^P)$.
If $u(\WFM(P)) = \{\}$ (i.e., the set of undefined atoms of the WFM of $P$ is empty), the WFM is two-valued and the program is called \textit{dynamically stratified}.
%
The WFS enjoys the property of relevance, and the SMS and WFS are related since, for a normal program $P$, the WFM of $P$ is a subset of every stable model of $P$ seen as a three-valued interpretation, as proven by~\cite{well-founded}.

\subsection{SLG Resolution and Tabling}
\label{subsec:tabling}
SLG resolution was proposed by~\cite{DBLP:journals/jacm/ChenW96} and was proven sound and complete for the WFS under certain conditions.
Its implementation in the most common Prolog systems, such as XSB~\citep{DBLP:journals/tplp/SwiftW12} and SWI~\citep{wielemaker2012swi}, is based on \textit{tabling}.
In the forest of tree model of SLG resolution~\citep{DBLP:conf/epia/Swift99}, a tree is generated for each sub-goal encountered during the derivation of a query.
Nodes are of the form $\textit{fail}$ or
$$
AnswerTemplate \impl GoalList | DelayList
$$
where $AnswerTemplate$ is a (partial) instantiation of the sub-goal and $GoalList$ and $DelayList$ are lists of literals.
$DelayList$ contains a set of literals that have been \emph{delayed}, which  is needed to allow the evaluation of a query under the WFS (where the computation cannot follow a fixed order for  literal selection) with a Prolog engine (where the computation follows a fixed order for selecting literals in a rule).
An answer is  a leaf with an empty $GoalList$. It is named \emph{unconditional} if the set of delayed atoms is empty; \emph{conditional} otherwise.

The XSB and SWI implementations of SLG allow mixing it with SLDNF resolution. To obtain the SLG behavior on a predicate, the user should declare the predicate as \textit{tabled} via the directive $table/1$, and use  $tnot$ instead of $not$ or $\backslash+$ to express negation.
After the full evaluation of a query, a forest of trees is built where each leaf node is either $\textit{fail}$ or an answer. 
If there are conditional answers, we also get the \emph{residual program} $P^r_q$, i.e., the program where the head of rules have the $AnswerTemplate$ of answer nodes and the body contains
the literals of the delay list. SLG resolution is sound and complete with respect to the WFS in the sense that atoms that are instantiations of unconditional answers have value true, those that
are instantiations of sub-goals whose tree has only $\mathit{fail}$ leaves are false, and those that are instantiations of conditional answers are undefined in the WFM of the program.
Let us now show an example.
\begin{example}
\label{ex:tnot_example}
The following program defines three tabled predicates.
\begin{lstlisting}
:- table edge/2.
:- table nedge/2.
:- table path/2.
e(a,b). e(a,c). e(b,d).
edge(A,B):- e(A,B), tnot(nedge(A,B)).
nedge(A,B):- e(A,B), tnot(edge(A,B)).
path(A,B):- edge(A,B).
path(A,B):- edge(A,C), path(C,B).
\end{lstlisting}
If we query $path(a,d)$, we get the following residual program:
\begin{lstlisting}
path(a,d) :- path(b,d), edge(a,b).
path(b,d) :- edge(b,d).
edge(b,d) :- tnot(nedge(b,d)).
nedge(b,d) :- tnot(edge(b,d)).
edge(a,b) :- tnot(nedge(a,b)).
nedge(a,b) :- tnot(edge(a,b)).
\end{lstlisting}
\end{example}

\subsection{Probabilistic Answer Set Programming}
\label{subsec:pasp}
The Credal Semantics (CS)~\citep{cozman2020pasp} allows the representation of uncertain domains with ASP extended with ProbLog probabilistic facts~\citep{DBLP:conf/ijcai/RaedtKT07} of the form $p_i::a_i$ where $p_i \in [0,1]$ is a probability and  $a_i$ is a ground atom.
Such programs are called Probabilistic ASP (PASP, and we use the same acronym to denote Probabilistic Answer Set Programs).
A \textit{world} is obtained by adding to the ASP a subset of the atoms $a_i$ where $p_i::a_i$ is a probabilistic fact.
Every PASP with $n$ probabilistic facts has thus $2^n$ worlds.
The probability of a world $w$ is computed as:
$P(w) = \prod_{a_i \in w} p_i \cdot \prod_{a_i \not\in w} (1 - p_i)$.
Each world is an ASP and it may have 0 or more answer sets but, for the CS to be defined, it is required that each world has at least one AS. 
If the ASP is normal without OLON, then the CS exists.

In regular Probabilistic Logic Programming (PLP)~\citep{Rig23-BKaddress}, each world is assigned a WFM which is required to be two-valued, and the probability $P(q)$ of a ground literal (called \emph{query}) $q$ is given by the sum of the  probabilities of the worlds where $q$ is true. In PASP, each world may have more than one AS, so the question is: how to distribute the probability mass of a world among its AS?
The CS answers this question by not assuming a specific distribution but allowing all the possible ones, leading to the association of a probability interval to $q$.
The upper bound $\upperprob(q)$ and the lower bound $\lowerprob(q)$ are given by:
\begin{equation}
\label{eq:upper_lower_prob}
\upperprob(q) = \sum_{w_i \mid \exists m \in AS(w_i), \ m \models q} P(w_i),
\ \ \ 
\lowerprob(q) = \sum_{w_i \mid \forall m \in AS(w_i), \ m \models q} P(w_i).
\end{equation}
In other words, the upper probability is the sum of the probabilities of the worlds where the query is present in at least one answer set, while the lower probability is the sum of the probabilities of the worlds where the query is present in every answer set.
In the case that every world has exactly one answer set, the lower and upper probability coincide, otherwise $\upperprob(q) > \lowerprob(q)$.

To clarify, consider the following example.
\begin{example}
\label{ex:pasp_example}
The following PASP defines 3 probabilistic facts.
\begin{lstlisting}
0.1::e(a,b). 0.2::e(a,c). 0.3::e(b,d).
edge(A,B):- e(A,B), not nedge(A,B).
nedge(A,B):- e(A,B), not edge(A,B).
path(A,B):- edge(A,B).
path(A,B):- edge(A,C), path(C,B).
\end{lstlisting}
It has $2^3 = 8$ worlds, listed in Table~\ref{tab:worlds_pasp}.
Consider the query $path(a,d)$.
Call $w_5$ the world where $e(a,b)$ and $e(b,d)$ are present and $e(a,c)$ absent.
It has 4 answer sets but only one of them includes the query.
Call $w_7$ the world where all the probabilistic facts are true.
It has 8 answer sets but only two include the query.
Overall, the probability of the query is
$[0,P(w_5) + P(w_7)] = [0,0.03]$.
Note that in $w_5$ and $w_7$ the query is present only in some answer sets, so they contribute only to the upper probability. 
\end{example}

\begin{table}
\centering
\caption{Worlds and probabilities for Example~\ref{ex:pasp_example}.
The column \#q/\# A.S. contains the number of answer sets where the query $path(a,d)$ is true and the total number of answer sets.}
\label{tab:worlds_pasp}
{\tablefont\begin{tabular}{@{\extracolsep{\fill}}c  c  c  c  c  c }
    \topline
id & $e(a,b)$ & $e(a,c)$ & $e(b,d)$ & \#q/\# A.S. & Probability
\midline
$w_0$ & 0 & 0 & 0 & 0/0 & $(1-0.1) \cdot (1-0.2) \cdot (1-0.3) = 0.504$ \\
$w_1$ & 0 & 0 & 1 & 0/2 & $(1-0.1) \cdot (1-0.2) \cdot 0.3 = 0.216$ \\
$w_2$ & 0 & 1 & 0 & 0/2 & $(1-0.1) \cdot 0.2 \cdot (1-0.3) = 0.126$ \\ 
$w_4$ & 0 & 1 & 1 & 0/4 & $(1-0.1) \cdot 0.2 \cdot 0.3 = 0.054$ \\
$w_3$ & 1 & 0 & 0 & 0/2 & $0.1 \cdot (1-0.2) \cdot (1-0.3) = 0.056$ \\
$w_5$ & 1 & 0 & 1 & 1/4 & $0.1 \cdot (1-0.2) \cdot 0.3 = 0.024$ \\ 
$w_6$ & 1 & 1 & 0 & 0/4 & $0.1 \cdot 0.2 \cdot (1-0.3) = 0.014$ \\
$w_7$ & 1 & 1 & 1 & 2/8 & $0.1 \cdot 0.2 \cdot 0.3 = 0.006$ 
\botline
\end{tabular}
}
\end{table}

Several approaches exist to perform inference in PASP, such as projected answer set enumeration~\citep{AzzBellRig2022PASTA} or Second Level Algebraic Model Counting (2AMC)~\citep{DBLP:journals/tplp/KieselTK22}.
Here we focus on the latter since it has been proved more effective~\citep{AzzRig2023-AIXIA-IC}.
The components of a 2AMC problem are:
a propositional theory $T$ whose variables are divided into two disjoint sets, $X_o$ and $X_i$, 
two commutative semirings $R^{i} = (D^i, \oplus^i, \otimes^i, n_{\oplus^i}, n_{\otimes^i})$ and $R^{o} = (D^o, \oplus^o, \otimes^o, n_{\oplus^o}, n_{\otimes^o})$,
two weight functions, $w_i: lit(X_i)\rightarrow D^i$ and $w_o: lit(X_o)\rightarrow D^o$, 
and a transformation function $f: D^i\rightarrow D^o$, where $lit(X)$ is the set of literals build on the variables from $X$.
2AMC is represented as:
\begin{equation}
  \label{eq:2amc}
  \begin{split}  
    2AMC(T) =& 
    \bigoplus\nolimits_{I_{o} \in \mu(X_{o})}^{o} 
    \bigotimes\nolimits^{o}_{a \in I_{o}} 
    w_{o}(a) 
    \otimes^{o}
    f(
      \bigoplus\nolimits_{I_{i} \in \varphi(\Pi \mid I_{o})}^{i} \bigotimes\nolimits^{i}_{b \in I_{i}} w_{i}(b)  
    )  
  \end{split}
\end{equation}
where $\mu(X_{o})$ is the set of assignments to the variables in $X_{o}$ and $\varphi(T \mid I_{o})$ is the set of assignments to the variables in $T$ that satisfy $I_{o}$.
In other words, the 2AMC task requires solving an Algebraic Model Counting (AMC)~\citep{10.1016/j.jal.2016.11.031} task on the variables $X_i$ for each assignment of the variables $X_o$.
The probability of a query $q$ in a PASP $P$ can be computed by translating $P$ into a propositional theory $T$ with exactly the same models as the AS of $P$ 
and using~\citep{AzzRig2023-AIXIA-IC} 
i) $\mathcal{R}^{i} = (\mathbb{N}^2, +, \cdot, (0,0), (1,1))$ with $w_i$ mapping $not \ q$ to $(0, 1)$ and all other literals to $(1, 1)$;
ii) as transformation function $f(n_1,n_2)$ computing $(v_{lp},v_{up})$ where $v_{lp} = 1$ if $n_1 = n_2$, 0 otherwise, and $v_{up} = 1$ if $n_1 > 0$, 0 otherwise, and, 
iii) $\mathcal{R}^{o} = ([0, 1]^2, +, \cdot,(0, 0),(1, 1))$, with $w_o$ associating $(p, p)$ and $(1 - p, 1 - p)$ to $a$ and $not \ a$, respectively, for every probabilistic fact $p :: a$ and $(1, 1)$ to all the remaining literals.
Here $X_o$ contains all the probabilistic atoms and $X_i$ contains the remaining atoms from $B_P$.
aspmc~\citep{DBLP:conf/kr/EiterHK21} is a tool that can solve 2AMC via knowledge compilation~\citep{DBLP:journals/jair/DarwicheM02} targeting negation normal form (NNF) formulas.
One measure to assess the size of a program, and so its complexity, is the treewidth~\citep{bodlaender1993tourist}, which represents how distant is the considered formula graph from being a tree.
The treewidth of a graph can be obtained via tree decomposition, a process that generates a tree starting from a graph, where nodes are assigned to bags, i.e., subsets of vertices.
A graph may have more than one tree decomposition but its treewidth $t$ is the minimum integer $t$ such that there exists a tree decomposition whose bags have size at most $t+1$.

\section{Extracting the Residual Program for PASP}
\label{sec:residual_program_extraction}

From Example~\ref{ex:pasp_example}, we can see that the probability of the query $path(a,d)$ is not influenced by the probabilistic fact $e(a,c)$.
Let us call $P(e(a,b)) = p_0$, $P(e(a,c)) = p_1$, and $P(e(b,d)) = p_2$, for brevity.
The upper probability of $path(a,d)$ is computed as $P(w_5) + P(w_7) = p_0 \cdot (1-p_1) \cdot p_2 +  p_0 \cdot p_1 \cdot p_2 = (p_0 \cdot p_2) \cdot ((1-p_1) + p_1) = p_0 \cdot p_2$, so, the value of $p_1$ is irrelevant, and the probabilistic fact $e(a,c)$ can be removed from the program.
However, during the grounding process, the probabilistic fact is still considered, increasing the size of the grounding.
The same happens with rules that do not influence the probability of a query.
While the programmer should take care of writing a compact program, encoding exactly the minimal information needed to answer a query, this is usually difficult to do.
Consider again Example~\ref{ex:pasp_example}: here it is difficult to immediately spot that $e(a,c)$ is irrelevant to the computation of the probability of $path(a,d)$.
To overcome this, the PLP systems PITA~\citep{riguzzi2011pita} and ProbLog2~\citep{dries2015problog2} build a proof for a query containing only the rules that are actually involved in the probability computation.
This is possible in PLP since it enjoys the property of relevance.
However, the SMS for normal programs without OLON also enjoys the property of relevance, so we aim to do the same in PASP by exploiting the residual program.

We first provide a definition and two results regarding the residual program.

\begin{definition}
\label{def:wf-red}
Given a normal program $P$, the $\mathit{WF}$ reduct of $P$, indicated with $P^\mathit{WF}$, is obtained by removing from $ground(P)$ the rules with the body false in $\WFM(P)$ and by removing from the body of the remaining rules the literals that are true in $\WFM(P)$.
\end{definition}
\begin{lemma}
\label{lem:wf-red-equiv}
Given a normal program $P$, $AS(P)=AS(P^\mathit{WF})$.
\end{lemma}
\begin{proof}
Consider an $A\in AS(P)$.
Then $t(\WFM(P))\subseteq A$ and $f(\WFM(P))\subseteq (B_P\setminus A)$ (see Section~\ref{subsec:wfs}).
Consider a rule $r\in P^A$ (i.e., the reduct of $P$ w.r.t. $A$).
Then $(P^\mathit{WF})^A$ contains a rule $r'$ that differs from $r$ because the body does not contain literals that are true in all answer sets and so also in $A$.
Since $r$ is satisfied in $A$, $r'$ is also satisfied in $A$. So $A$ is a model of $(P^\mathit{WF})^A$.
Moreover, $A$ is also a minimal model of $(P^\mathit{WF})^A$, because otherwise there would be at least one atom $a$ that could be removed from $A$ leading to a set $A'$ that would still be a model for $(P^\mathit{WF})^A$.
However, since $A$ was minimal for $P^A$, this means that there is a rule $r=a\impl body$ with $body$ true in $A$.
Since there would be a rule $r'=a\impl body'$ in $(P^\mathit{WF})^A$ with $body'$ still true, then $a$ cannot be removed from $A$ against the hypothesis.
So $A\in AS(P^\mathit{WF})$.

On the other hand, consider an $A\in AS(P^\mathit{WF})$ and a rule $r \in (P^\mathit{WF})^A$.
Then $ground(P)$ contains a rule $r'$ that differs from $r$ because the body contains other literals that
are true in all answer sets of $P$. Since $r$ is satisfied in $A$, $r'$ is also satisfied in $A$. So $A$ is a model of $P^A$. Moreover, $A$ is also minimal, because otherwise there would be at least one atom $a$ that could be removed from $A$ leading to a set $A'$ that would still be a model for $P^A$.
However, since $A$ was minimal for $(P^\mathit{WF})^A$, this means that there is a rule $r=a\impl body$ with $body$ true in $A$.
Since there would be a rule $r'=a\impl body'$ in $(P^\mathit{WF})^A$ with $body'$ still true, then $a$ cannot be removed from $A$ against the hypothesis. So $A\in AS(P)$.
\end{proof}

\begin{theorem}
\label{th:equivalence_projected_as}
Given a normal program $P$ without OLON together with its residual program $P^r_q$ for a query $q$, the  answer sets projected  onto the Herbrand base $B_{P^r_q}$ of $P^r_q$ coincide with the answer sets of $P^r_q$, i.e.,
$$
AS_{B_{P^r_q}}(P) = AS(P^r_q).
$$
\end{theorem}
\begin{proof}
$AS( P^\mathit{WF})=AS(P)$ by Lemma~\ref{lem:wf-red-equiv}, so we  prove that
$
AS_{B_{P^r_q}}(P^\mathit{WF}) = AS(P^r_q)
$.
For the soundness of SLG resolution and the fact it analyses the whole relevant sub-graph, the truth of the body of each rule $r\in P^r_q$ is not influenced by the truth value of atoms outside $B_{P^r_q}$.
Therefore, an $A\in AS(P^r_q)$ can be extended to an $A'\in AS(P^\mathit{WF})$ such that $A=A'\cap B_{P^r_q}$.
Thus, $A\in AS_{B_{P^r_q}}(P^\mathit{WF})$.
In the other direction, if $A'\in AS(P^\textit{WF})$, consider $A=A'\cap B_{P^r_q}$.
Since $B_{P^r_q}$ contains all the atoms in the relevant sub-graph, the truth of the body of each rule $r\in P^r_q$ is not influenced by the truth value of atoms outside $B_{P^r_q}$ and $A$ must be an AS of $P^r_q$.
\end{proof}
To consider PASP, we first translate a PASP into a normal program.
We convert each probabilistic fact $p::a$ into a pair of rules:
\begin{lstlisting}
a :- tnot(na).
na :- tnot(a).
\end{lstlisting} 
where $na$ is a fresh atom not appearing elsewhere in the program.
This pair of rules encode the possibility that a probabilistic fact may or may not be selected.
Then, we replace the negation symbol applied to each atom $b$ with $tnot(b)$ and declare as tabled all the predicates appearing in the program.
We extract the residual program and we replace the pair of rules mimicking probabilistic facts with the actual probabilistic fact they represent (i.e., the two rules listed in the previous box are replaced with $p::a$).
Then, we call a standard solver such as aspmc~\citep{DBLP:conf/kr/EiterHK21,EITER2024104109} or PASTA~\citep{AzzBellRig2022PASTA}.

\begin{theorem}
Given a PASP $P$ together with its residual program $P^r_q$ for a query $q$, let $\overline{P}(q)$ be the upper probability of $q$ in $P$ and $\overline{P'}(q)$ be the upper probability of $q$ in $P^r_q$.
Then 
$$
\overline{P}(q)=\overline{P'}(q).
$$
The same is true for the lower probability.
\end{theorem}
\begin{proof}
If the clauses generated for some probabilistic fact are absent from the residual program, they will not influence the probability.
Let us prove it by induction on the number $n$ of probabilistic facts whose clauses are absent.
If $n=1$ and the fact is $p_1::a_1$, consider an AS $A$ in $AS_{B_{P^r_q}}(P^r_q)$, associated to a world $w$ of $P^r_q$.
Then, there are going to be two subsets of $AS(P)$, ${\cal A}'$ and ${\cal A}''$, such that
$\forall I\in {\cal A}'\cup{\cal A}'':I\supseteq A$, $\forall I\in {\cal A}':a_1\in I$ and $\forall I\in{\cal A}'':na_1\in I$.
${\cal A}'$ is the set of AS of a world $v'$ such that $a_1\in v'$ and 
${\cal A}''$ is the set of AS of a world $v''$ such that $na_1\in v''$.
However, if $q\in A$, then $\forall I\in{\cal A}'\cup{\cal A}'':q\in I$, so $v'$ and $v''$ either both contribute to one of the probability bounds or neither does. 
The contribution, if present, would be given by $P(w)\cdot p_1+P(w)\cdot (1-p_1)=P(w)$, so the fact  $p_1::a_1$ does not influence the probability of $q$.
Now suppose the theorem holds for $n-1$ probabilistic facts whose clauses are not present and consider the $n$-th fact  $p_n::a_n$.
Let us call $P^*$ the program $P$ without the fact $p_n::a_n$.
Then $(P^*)^r_q=P^r_q$ and we can repeat the reasoning for $n=1$.
\end{proof}

\section{Experiments}
\label{sec:experiments}
We ran the experiments on a computer running at 2.40 GHz with 32 GB of RAM with cutoff times of 100, 300, and 500 seconds\footnote{Implementation and datasets are available at: \url{https://github.com/damianoazzolini/aspmc}}.

\subsection{Datasets Description}
\label{subsec:datasets_description}
We considered two datasets with two variations each and with an increasing number of instances.
The reachability (\textit{reach}) dataset models a reachability problem in a probabilistic graph.
All the instances have the rules (here we model negation with $\backslash+$, since it is the symbol adopted in aspmc):
\begin{lstlisting}
edge(X,Y):- e(X,Y), \+ nedge(X,Y).
nedge(X,Y):- e(X,Y), \+ edge(X,Y).
path(X,Y) :- edge(X,Y).
path(X,Z) :- edge(X,Y), path(Y,Z).
\end{lstlisting}
where the $e/2$ facts are probabilistic with probability 0.1.
We developed two variations for this dataset: \textit{reachBA} and \textit{reachGrid}.
The difference between the two is in the generation of the $e/2$ facts: for the former, they are generated by following a Barabasi-Albert model with initial number of nodes equal to the size of the instance and 2 edges to attach from a new node to existing nodes (these two values are respectively the values of the $n$ and $m$ parameters of the method \texttt{barabasi\_albert\_graph} of the NetworkX Python library~\citep{hagberg2008networkx} we used to generate them).
The query is $path(0,n-1)$.
For the latter, the $e/2$ facts are such that they form a two-dimensional grid.
In this case, the query is $path(0,i)$, where $i$ is a random node (different for every dataset).

The \textit{smokers} dataset contains a set of programs modeling a social network where some people smoke and others are influenced by this behavior.
Each person is indexed with a number, starting from 0.
The base program is:
\begin{lstlisting}
influences(X,Y):- e(X,Y), \+ ninfluences(X,Y).
ninfluences(X,Y):- e(X,Y), \+ influences(X,Y).
smokes(X) :- stress(X).
smokes(X) :- smokes(Y), influences(Y,X).
\end{lstlisting}
Each $stress/1$ atom is probabilistic with probability 0.1 and each $\mathit{influences}/2$ atom is probabilistic with associated probability 0.2.
Also for this dataset we consider two variations, \textit{smokersBA} and \textit{smokersGrid}, that are generated with the same structure as for \textit{reachBA} and \textit{reachGrid}, respectively.
For \textit{smokersBA} the query is $smokes(n-1)$ where $n$ is the number of person in the network, while for \textit{smokersGrid} the query is $smokes(i)$ where $i$ is a random person (different for every dataset).
For all the instances, the probability associated with probabilistic facts does not influence the time required to compute the probability of the query.

\subsection{Results}
\label{subsec:results}

\begin{table}
\centering
\caption{Values for the \textit{reachBA} and \textit{smokersBA} datasets in the format (\aspmcr - aspmc) for the tests with 500 seconds of time limit.
$\mu$ stands for (rounded) mean, tw. for treewidth, and vert. for vertices related to the dependency graph.
The column \# unsolved contains the number of unsolved instances for the specific size.}
\label{tab:tab_ba}
{\tablefont\begin{tabular}{@{\extracolsep{\fill}}c || c | c | c | c ||c | c | c | c }
\topline
\multicolumn{1}{c}{} & \multicolumn{4}{c}{\textit{reachBA} \aspmcr - aspmc} & \multicolumn{4}{c}{\textit{smokersBA} \aspmcr - aspmc} \\
\hline
size \ & \ \# unsolved \ & \ $\mu$ \# bags \ & \ $\mu$ tw. \ & \ $\mu$ \# vert. \ & \ \# unsolved \ & \ $\mu$ \# bags \ & \ $\mu$ tw. \ & \ $\mu$ \# vert. \
\midline
5   & 0 - 0  &\ 21 - 36   \ & \ 5 - 6 \     & \ 30 -  48  \     & 0 - 0   & \  29 - 38  \   & \ 9 - 12  \    & \ 44 - 55    \\ 
10  & 0 - 0  &\ 37 - 121  \ & \ 8 - 16 \    & \ 52 -  156  \    & 0 - 10  & \  40 - 86  \   & \ 13 - 27  \   & \ 65 - 135   \\ 
15  & 0 - 10 &\ 46 - 205  \ & \ 11 - 26 \   & \ 66 -  263  \    & 2 - 10  & \  38 - 113  \  & \ 14 - 42  \   & \ 71 - 215   \\ 
20  & 1 - 10 &\ 46 - 331  \ & \ 14 - 36 \   & \ 86 -  415  \    & 0 - 10  & \  47 - 169  \  & \ 17 - 57  \   & \ 84 - 295    \\ 
25  & 0 - 10 &\ 56 - 441  \ & \ 13 - 46 \   & \ 79 -  550  \    & 4 - 10  & \  55 - 257  \  & \ 19 - 72  \   & \ 98 - 375    \\ 
30  & 0 - 10 &\ 48 - 564  \ & \ 15 - 56 \   & \ 95 -  700  \    & 4 - 10  & \  44 - 259  \  & \ 21 - 87  \   & \ 104 - 455  \\ 
35  & 0 - 10 &\ 49 - 695  \ & \ 13 - 67 \   & \ 81 -  859  \    & 3 - 10  & \  62 - 336  \  & \ 21 - 102  \  & \ 106 - 535   \\ 
40  & 1 - 10 &\ 54 - 861  \ & \ 16 - 76 \   & \ 101 -  1056  \  & 2 - 10  & \  52 - 384  \  & \ 20 - 117  \  & \ 101 - 614   \\ 
45  & 3 - 10 &\ 42 - 1006 \ & \ 21 - 86 \   & \ 131 -  1233  \  & 5 - 10  & \  69 - 390  \  & \ 24 - 132  \  & \ 121 - 695  \\ 
50  & 2 - 10 &\ 42 - 1163 \ & \ 17 - 96 \   & \ 108 -  1426  \  & 6 - 10  & \  74 - 532  \  & \ 25 - 147  \  & \ 124 - 774   \\ 
55  & 2 - 10 &\ 56 - 1253 \ & \ 16 - 106 \  & \ 101 -  1539  \  & 5 - 10  & \  61 - 587  \  & \ 22 - 162  \  & \ 113 - 855   \\ 
60  & 3 - 10 &\ 52 - 1480 \ & \ 21 - 116 \  & \ 136 -  1808  \  & 8 - 10  & \  98 - 641  \  & \ 30 - 177  \  & \ 154 - 935   \\ 
65  & 1 - 10 &\ 43 - 1629 \ & \ 15 - 127 \  & \ 91 -  1978  \   & 6 - 10  & \  66 - 696  \  & \ 24 - 192  \  & \ 123 - 1014  \\ 
70  & 5 - 10 &\ 50 - 1811 \ & \ 20 - 136 \  & \ 124 -  2198  \  & 9 - 10  & \  109 - 752  \ & \ 33 - 207  \  & \ 169 - 1095  \\ 
75  & 5 - 10 &\ 73 - 1906 \ & \ 22 - 146 \  & \ 139 -  2311  \  & 8 - 10  & \  58 - 807  \  & \ 33 - 222  \  & \ 171 - 1175  \\ 
80  & 0 - 10 &\ 47 - 1952 \ & \ 14 - 156 \  & \ 84 -  2383  \   & 6 - 10  & \  49 - 861  \  & \ 28 - 237  \  & \ 144 - 1254  \\
85  & 5 - 10 &\ 32 - 2429 \ & \ 24 - 167 \  & \ 152 -  2932  \  & 3 - 10  & \  60 - 915  \  & \ 21 - 252  \  & \ 107 - 1332  \\ 
90  & 1 - 10 &\ 61 - 2431 \ & \ 20 - 176 \  & \ 128 -  2944  \  & 7 - 10  & \  83 - 970  \  & \ 29 - 267  \  & \ 147 - 1414  \\ 
95  & 5 - 10 &\ 53 - 2519 \ & \ 20 - 186 \  & \ 130 -  3058  \  & 8 - 10  & \  91 - 1025  \ & \ 30 - 282  \  & \ 152 - 1493  \\ 
100 & 3 - 10 &\ 44 - 2795 \ & \ 19 - 197 \  & \ 119 -  3385  \  & 9 - 10  & \  97 - 1081  \ & \ 34 - 297  \  & \ 178 - 1575 
\end{tabular}
}
\end{table}

\begin{table}
\centering
\caption{Values for the \textit{reachGrid} and \textit{smokersGrid} datasets in the format (\aspmcr - aspmc) for the tests with 500 seconds of time limit.
$\mu$ stands for (rounded) mean, tw. for treewidth, and vert. for vertices related to the dependency graph.
The column \# unsolved contains the number of unsolved instances for the specific size.}
\label{tab:tab_grid}
{\tablefont\begin{tabular}{@{\extracolsep{\fill}}c || c | c | c | c ||c | c | c | c }
\topline
\multicolumn{1}{c}{} & \multicolumn{4}{c}{\textit{reachGrid} \aspmcr - aspmc} & \multicolumn{4}{c}{\textit{smokersGrid} \aspmcr - aspmc} \\
\hline
size \ & \ \# unsolved \ & \ $\mu$ \# bags \ & \ $\mu$ tw. \ & \ $\mu$ \# vert. \ & \ \# unsolved \ & \ $\mu$ \# bags \ & \ $\mu$ tw. \ & \ $\mu$ \# vert. \
\midline
2 & 0 - 0   & \ 10 -  25 \ & \ 4 -  5 \ & \ 15 - 33 \  & 0 - 0  & \ 11 - 27 \   & \ 4 - 9 \     & \ 15 - 39    \\ 
3 & 3 - 0   & \ 20 -  98 \ & \ 5 -  13 \ & \ 28 - 124 \  & 1 - 0  & \ 15 - 22 \   & \ 8 - 22 \    & \ 37 - 105   \\ 
4 & 1 - 10  & \ 25 -  274 \  & \ 10 -  27 \ & \ 62 - 326 \  & 3 - 10 & \ 17 - 24 \   & \ 17 - 41 \   & \ 82 - 203   \\ 
5 & 3 - 10  & \ 31 -  592 \  & \ 14 -  46 \ & \ 88 - 701 \  & 0 - 10 & \ 15 - 27 \   & \ 7 - 66 \    & \ 30 - 333   \\ 
6 & 6 - 10  & \ 27 -  1160 \ & \ 28 -  68 \ & \ 178 - 1337 \  & 3 - 10 & \ 21 - 34 \   & \ 26 - 98 \   & \ 124 - 495  \\ 
7 & 6 - 10  & \ 28 -  2007 \ & \ 36 -  98 \ & \ 231 - 2333 \  & 5 - 10 & \ 23 - 27 \   & \ 32 - 137 \  & \ 158 - 689 \\ 
8 & 6 - 10  & \ 39 -  3305 \ & \ 38 -  127 \ & \ 242 - 3810 \  & 7 - 10 & \ 24 - 608 \  & \ 49 - 186 \  & \ 245 - 915   \\ 
9 & 8 - 10  & \ 56 -  5150 \ & \ 46 -  162 \ & \ 292 - 5906 \  & 6 - 10 & \ 26 - 779 \  & \ 50 - 236 \  & \ 250 - 1173  \\ 
10 & 5 - 10  & \ 35 -  7679 \ & \ 43 -  208 \ & \ 273 - 8777 \ & 8 - 10 & \ 263 - 972 \ & \ 113 - 293 \ & \ 562 - 1463  \\ 
\end{tabular}
}

\end{table}

\begin{figure}
\centering
\begin{subfigure}{0.48\textwidth}
\centering
\resizebox{\resizeGraphFactor\textwidth}{!}{%
\begin{tikzpicture}
\begin{axis}[
    xlabel={Number of Solved Instances},
    ylabel={Cumulative Time (s)},
    legend pos=north west,
    legend style={fill=white, fill opacity=0.6},
    grid style=dashed,
    legend cell align={left},every axis plot/.append style={thick}
    ]

\addplot[color = original_plot] 
    coordinates {
        (0, 0)
        (1, 2)
        (2, 5)
        (3, 8)
        (4, 11)
        (5, 14)
        (6, 17)
        (7, 20)
        (8, 23)
        (9, 26)
        (10, 29)
        (11, 1011)
    };
\addlegendentry{aspmc 100s}
\addplot[color = original_plot, dashed] 
    coordinates {
        (0, 0)
        (1, 3)
        (2, 6)
        (3, 10)
        (4, 13)
        (5, 17)
        (6, 20)
        (7, 24)
        (8, 28)
        (9, 32)
        (10, 36)
        (11, 1011)
    };
\addlegendentry{aspmc 300s}
\addplot[color = original_plot, dotted] 
    coordinates {
        (0, 0)
        (1, 3)
        (2, 6)
        (3, 10)
        (4, 13)
        (5, 17)
        (6, 20)
        (7, 24)
        (8, 28)
        (9, 32)
        (10, 35)
        (11, 114)
        (12, 193)
        (13, 275)
        (14, 358)
        (15, 444)
        (16, 530)
        (17, 617)
        (18, 706)
        (19, 795)
        (20, 885)
        (21, 1011)
    };
\addlegendentry{aspmc 500s}
\addplot[color = tabled_plot] 
    coordinates {
        (0, 0)
        (1, 2)
        (2, 5)
        (3, 8)
        (4, 11)
        (5, 14)
        (6, 17)
        (7, 20)
        (8, 23)
        (9, 26)
        (10, 29)
        (11, 32)
        (12, 35)
        (13, 38)
        (14, 41)
        (15, 44)
        (16, 47)
        (17, 50)
        (18, 53)
        (19, 57)
        (20, 60)
        (21, 63)
        (22, 66)
        (23, 69)
        (24, 72)
        (25, 75)
        (26, 78)
        (27, 81)
        (28, 84)
        (29, 87)
        (30, 91)
        (31, 94)
        (32, 97)
        (33, 100)
        (34, 104)
        (35, 107)
        (36, 110)
        (37, 114)
        (38, 117)
        (39, 121)
        (40, 124)
        (41, 127)
        (42, 131)
        (43, 134)
        (44, 138)
        (45, 142)
        (46, 146)
        (47, 149)
        (48, 156)
        (49, 162)
        (50, 175)
        (51, 188)
        (52, 203)
        (53, 284)
        (54, 365)
        (55, 459)
        (56, 1011)
    };
\addlegendentry{\aspmcr 100s}
\addplot[color = tabled_plot, dashed] 
    coordinates {
        (0, 0)
        (1, 2)
        (2, 5)
        (3, 8)
        (4, 11)
        (5, 14)
        (6, 17)
        (7, 20)
        (8, 23)
        (9, 26)
        (10, 29)
        (11, 32)
        (12, 35)
        (13, 38)
        (14, 41)
        (15, 44)
        (16, 47)
        (17, 50)
        (18, 53)
        (19, 56)
        (20, 59)
        (21, 62)
        (22, 65)
        (23, 69)
        (24, 72)
        (25, 75)
        (26, 78)
        (27, 81)
        (28, 84)
        (29, 88)
        (30, 91)
        (31, 94)
        (32, 97)
        (33, 101)
        (34, 104)
        (35, 108)
        (36, 111)
        (37, 115)
        (38, 118)
        (39, 122)
        (40, 125)
        (41, 129)
        (42, 132)
        (43, 136)
        (44, 139)
        (45, 143)
        (46, 147)
        (47, 153)
        (48, 159)
        (49, 165)
        (50, 176)
        (51, 189)
        (52, 257)
        (53, 329)
        (54, 409)
        (55, 683)
        (56, 1011)
    };

\addlegendentry{\aspmcr 300s}
\addplot[color = tabled_plot, dotted] 
    coordinates {
        (0, 0)
        (1, 2)
        (2, 5)
        (3, 8)
        (4, 11)
        (5, 14)
        (6, 17)
        (7, 20)
        (8, 23)
        (9, 26)
        (10, 29)
        (11, 32)
        (12, 35)
        (13, 38)
        (14, 41)
        (15, 44)
        (16, 47)
        (17, 50)
        (18, 53)
        (19, 56)
        (20, 59)
        (21, 62)
        (22, 65)
        (23, 68)
        (24, 71)
        (25, 74)
        (26, 77)
        (27, 80)
        (28, 83)
        (29, 86)
        (30, 90)
        (31, 93)
        (32, 96)
        (33, 99)
        (34, 102)
        (35, 106)
        (36, 109)
        (37, 112)
        (38, 116)
        (39, 119)
        (40, 122)
        (41, 126)
        (42, 129)
        (43, 133)
        (44, 136)
        (45, 140)
        (46, 144)
        (47, 147)
        (48, 153)
        (49, 160)
        (50, 172)
        (51, 186)
        (52, 201)
        (53, 278)
        (54, 361)
        (55, 458)
        (56, 731)
        (57, 1011)
    };
\addlegendentry{\aspmcr 500s}

\end{axis}
\end{tikzpicture}
}
\caption{\textit{smokersGrid}}
\label{subfig:cactus_smk_grid}
\end{subfigure}%
\hfill
\begin{subfigure}{0.48\textwidth}
\resizebox{\resizeGraphFactor\textwidth}{!}{%
\begin{tikzpicture}
\begin{axis}[
    xlabel={Number of Solved Instances},
    ylabel={Cumulative Time (s)},
    legend pos=north west,
    legend style={fill=white, fill opacity=0.6},
    grid style=dashed,
    legend cell align={left},every axis plot/.append style={thick}
    ]

\addplot[color = original_plot] 
    coordinates {
        (0, 0)
        (1, 3)
        (2, 6)
        (3, 9)
        (4, 12)
        (5, 15)
        (6, 18)
        (7, 21)
        (8, 24)
        (9, 28)
        (10, 31)
        (11, 5009)
    };
\addlegendentry{aspmc 100s}
\addplot[color = original_plot, dashed] 
    coordinates {
        (0, 0)
        (1, 3)
        (2, 6)
        (3, 9)
        (4, 12)
        (5, 15)
        (6, 18)
        (7, 21)
        (8, 24)
        (9, 27)
        (10, 31)
        (11, 5009)
    };
    
\addlegendentry{aspmc 300s}
\addplot[color = original_plot, dotted] 
    coordinates {
        (0, 0)
        (1, 3)
        (2, 6)
        (3, 9)
        (4, 12)
        (5, 15)
        (6, 18)
        (7, 21)
        (8, 24)
        (9, 28)
        (10, 31)
        (11, 5009)    
    };
\addlegendentry{aspmc 500s}

\addplot[color = tabled_plot] 
    coordinates {
        (0, 0)
        (1, 2)
        (2, 5)
        (3, 8)
        (4, 11)
        (5, 14)
        (6, 17)
        (7, 20)
        (8, 23)
        (9, 27)
        (10, 30)
        (11, 33)
        (12, 36)
        (13, 39)
        (14, 42)
        (15, 45)
        (16, 48)
        (17, 51)
        (18, 54)
        (19, 57)
        (20, 60)
        (21, 63)
        (22, 67)
        (23, 70)
        (24, 73)
        (25, 76)
        (26, 79)
        (27, 83)
        (28, 86)
        (29, 89)
        (30, 92)
        (31, 96)
        (32, 99)
        (33, 102)
        (34, 106)
        (35, 109)
        (36, 113)
        (37, 116)
        (38, 120)
        (39, 124)
        (40, 127)
        (41, 131)
        (42, 135)
        (43, 140)
        (44, 144)
        (45, 149)
        (46, 153)
        (47, 158)
        (48, 162)
        (49, 167)
        (50, 172)
        (51, 176)
        (52, 181)
        (53, 186)
        (54, 190)
        (55, 195)
        (56, 203)
        (57, 210)
        (58, 218)
        (59, 226)
        (60, 234)
        (61, 242)
        (62, 251)
        (63, 259)
        (64, 270)
        (65, 282)
        (66, 294)
        (67, 307)
        (68, 320)
        (69, 334)
        (70, 349)
        (71, 365)
        (72, 381)
        (73, 397)
        (74, 414)
        (75, 431)
        (76, 466)
        (77, 504)
        (78, 544)
        (79, 584)
        (80, 625)
        (81, 665)
        (82, 706)
        (83, 749)
        (84, 793)
        (85, 840)
        (86, 889)
        (87, 941)
        (88, 997)
        (89, 1073)
        (90, 1156)
        (91, 1242)
        (92, 1332)
        (93, 1426)
        (94, 1524)
        (95, 5009)
    };
\addlegendentry{\aspmcr 100s}
\addplot[color = tabled_plot, dashed] 
    coordinates {
        (0, 0)
        (1, 2)
        (2, 5)
        (3, 8)
        (4, 11)
        (5, 14)
        (6, 17)
        (7, 21)
        (8, 24)
        (9, 27)
        (10, 30)
        (11, 33)
        (12, 36)
        (13, 39)
        (14, 42)
        (15, 45)
        (16, 48)
        (17, 51)
        (18, 54)
        (19, 57)
        (20, 60)
        (21, 64)
        (22, 67)
        (23, 70)
        (24, 73)
        (25, 76)
        (26, 79)
        (27, 83)
        (28, 86)
        (29, 89)
        (30, 92)
        (31, 95)
        (32, 99)
        (33, 102)
        (34, 105)
        (35, 109)
        (36, 112)
        (37, 116)
        (38, 119)
        (39, 123)
        (40, 126)
        (41, 130)
        (42, 134)
        (43, 138)
        (44, 143)
        (45, 147)
        (46, 151)
        (47, 156)
        (48, 160)
        (49, 165)
        (50, 169)
        (51, 174)
        (52, 179)
        (53, 183)
        (54, 188)
        (55, 194)
        (56, 201)
        (57, 209)
        (58, 217)
        (59, 225)
        (60, 233)
        (61, 242)
        (62, 250)
        (63, 262)
        (64, 273)
        (65, 285)
        (66, 297)
        (67, 311)
        (68, 324)
        (69, 338)
        (70, 353)
        (71, 367)
        (72, 382)
        (73, 398)
        (74, 415)
        (75, 449)
        (76, 487)
        (77, 524)
        (78, 563)
        (79, 603)
        (80, 643)
        (81, 684)
        (82, 726)
        (83, 771)
        (84, 817)
        (85, 864)
        (86, 914)
        (87, 967)
        (88, 1036)
        (89, 1111)
        (90, 1186)
        (91, 1265)
        (92, 1357)
        (93, 1456)
        (94, 1556)
        (95, 1657)
        (96, 1761)
        (97, 1895)
        (98, 2194)
        (99, 5009)
    };
\addlegendentry{\aspmcr 300s}
\addplot[color = tabled_plot, dotted] 
    coordinates {
        (0, 0)
        (1, 2)
        (2, 5)
        (3, 8)
        (4, 11)
        (5, 14)
        (6, 17)
        (7, 20)
        (8, 23)
        (9, 27)
        (10, 30)
        (11, 33)
        (12, 36)
        (13, 39)
        (14, 42)
        (15, 45)
        (16, 48)
        (17, 51)
        (18, 54)
        (19, 57)
        (20, 60)
        (21, 63)
        (22, 66)
        (23, 70)
        (24, 73)
        (25, 76)
        (26, 79)
        (27, 82)
        (28, 85)
        (29, 89)
        (30, 92)
        (31, 95)
        (32, 98)
        (33, 102)
        (34, 105)
        (35, 109)
        (36, 112)
        (37, 116)
        (38, 119)
        (39, 123)
        (40, 127)
        (41, 130)
        (42, 134)
        (43, 139)
        (44, 143)
        (45, 148)
        (46, 152)
        (47, 156)
        (48, 161)
        (49, 165)
        (50, 170)
        (51, 174)
        (52, 179)
        (53, 184)
        (54, 189)
        (55, 193)
        (56, 201)
        (57, 208)
        (58, 216)
        (59, 224)
        (60, 232)
        (61, 241)
        (62, 249)
        (63, 260)
        (64, 270)
        (65, 283)
        (66, 295)
        (67, 308)
        (68, 322)
        (69, 336)
        (70, 350)
        (71, 365)
        (72, 380)
        (73, 397)
        (74, 414)
        (75, 445)
        (76, 478)
        (77, 515)
        (78, 555)
        (79, 595)
        (80, 635)
        (81, 676)
        (82, 721)
        (83, 768)
        (84, 816)
        (85, 865)
        (86, 922)
        (87, 989)
        (88, 1064)
        (89, 1139)
        (90, 1217)
        (91, 1296)
        (92, 1390)
        (93, 1487)
        (94, 1587)
        (95, 1724)
        (96, 1989)
        (97, 2278)
        (98, 2581)
        (99, 2885)
        (100, 3207)
        (101, 3532)
        (102, 3861)
        (103, 4231)
        (104, 4611)
        (105, 5009)        
    };
\addlegendentry{\aspmcr 500s}

\end{axis}
\end{tikzpicture}
}
\caption{\textit{smokersBA}.}
\label{subfig:cactus_smk_ba}
\end{subfigure}
\caption{Cactus plot for aspmc and \aspmcr with 100, 300, and 500 seconds of time on the \textit{smokersGrid} and \textit{smokersBA} datasets.}
\label{fig:cactus_smk}
\end{figure}
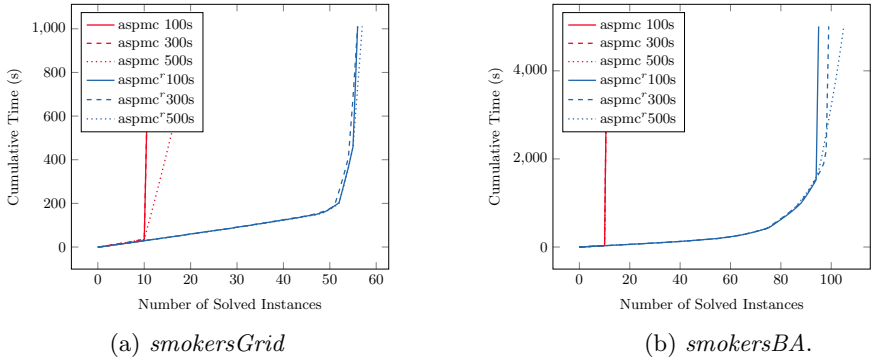

\begin{figure}
\centering
\begin{subfigure}{0.48\textwidth}
\centering
\resizebox{\resizeGraphFactor\textwidth}{!}{%
\begin{tikzpicture}
\begin{axis}[
    xlabel={Number of Solved Instances},
    ylabel={Cumulative Time (s)},
    legend pos=north west,
    legend style={fill=white, fill opacity=0.6},
    grid style=dashed,
    legend cell align={left},every axis plot/.append style={thick}
    ]

\addplot[color = original_plot] 
    coordinates {
        (0, 0)
        (1, 3)
        (2, 6)
        (3, 9)
        (4, 12)
        (5, 15)
        (6, 18)
        (7, 21)
        (8, 24)
        (9, 27)
        (10, 30)
        (11, 39)
        (12, 49)
        (13, 59)
        (14, 68)
        (15, 78)
        (16, 88)
        (17, 98)
        (18, 108)
        (19, 118)
        (20, 130)
        (21, 812)        
    };
\addlegendentry{aspmc 100s}
\addplot[color = original_plot, dashed] 
    coordinates {
        (0, 0)
        (1, 3)
        (2, 6)
        (3, 9)
        (4, 12)
        (5, 15)
        (6, 18)
        (7, 21)
        (8, 24)
        (9, 27)
        (10, 30)
        (11, 39)
        (12, 49)
        (13, 58)
        (14, 68)
        (15, 78)
        (16, 88)
        (17, 98)
        (18, 108)
        (19, 118)
        (20, 130)
        (21, 812)
    };
\addlegendentry{aspmc 300s}
\addplot[color = original_plot, dotted] 
    coordinates {
        (0, 0)
        (1, 2)
        (2, 5)
        (3, 8)
        (4, 11)
        (5, 14)
        (6, 17)
        (7, 20)
        (8, 22)
        (9, 25)
        (10, 28)
        (11, 38)
        (12, 47)
        (13, 56)
        (14, 66)
        (15, 75)
        (16, 85)
        (17, 95)
        (18, 105)
        (19, 115)
        (20, 127)
        (21, 812)        
    };
\addlegendentry{aspmc 500s}
\addplot[color = tabled_plot] 
    coordinates {
        (0, 0)
        (1, 2)
        (2, 5)
        (3, 8)
        (4, 11)
        (5, 14)
        (6, 17)
        (7, 20)
        (8, 23)
        (9, 26)
        (10, 29)
        (11, 32)
        (12, 35)
        (13, 38)
        (14, 41)
        (15, 44)
        (16, 47)
        (17, 50)
        (18, 53)
        (19, 56)
        (20, 59)
        (21, 62)
        (22, 65)
        (23, 69)
        (24, 72)
        (25, 75)
        (26, 78)
        (27, 81)
        (28, 84)
        (29, 87)
        (30, 91)
        (31, 94)
        (32, 97)
        (33, 101)
        (34, 104)
        (35, 107)
        (36, 111)
        (37, 114)
        (38, 118)
        (39, 121)
        (40, 125)
        (41, 128)
        (42, 132)
        (43, 136)
        (44, 140)
        (45, 144)
        (46, 150)
        (47, 180)
        (48, 210)
        (49, 269)
        (50, 366)
        (51, 466)
        (52, 812)
    };
\addlegendentry{\aspmcr 100s}
\addplot[color = tabled_plot, dashed] 
    coordinates {
        (0, 0)
        (1, 2)
        (2, 5)
        (3, 9)
        (4, 12)
        (5, 15)
        (6, 18)
        (7, 21)
        (8, 24)
        (9, 27)
        (10, 30)
        (11, 33)
        (12, 36)
        (13, 39)
        (14, 42)
        (15, 45)
        (16, 48)
        (17, 51)
        (18, 54)
        (19, 57)
        (20, 60)
        (21, 64)
        (22, 67)
        (23, 70)
        (24, 73)
        (25, 77)
        (26, 80)
        (27, 83)
        (28, 87)
        (29, 90)
        (30, 94)
        (31, 97)
        (32, 100)
        (33, 104)
        (34, 107)
        (35, 111)
        (36, 115)
        (37, 118)
        (38, 122)
        (39, 125)
        (40, 129)
        (41, 133)
        (42, 137)
        (43, 140)
        (44, 144)
        (45, 148)
        (46, 155)
        (47, 182)
        (48, 209)
        (49, 269)
        (50, 368)
        (51, 488)
        (52, 812)
    };
\addlegendentry{\aspmcr 300s}
\addplot[color = tabled_plot, dotted] 
    coordinates {
        (0, 0)
        (1, 2)
        (2, 5)
        (3, 8)
        (4, 11)
        (5, 14)
        (6, 17)
        (7, 20)
        (8, 23)
        (9, 26)
        (10, 29)
        (11, 32)
        (12, 35)
        (13, 38)
        (14, 41)
        (15, 44)
        (16, 47)
        (17, 50)
        (18, 53)
        (19, 56)
        (20, 59)
        (21, 62)
        (22, 66)
        (23, 69)
        (24, 72)
        (25, 75)
        (26, 78)
        (27, 81)
        (28, 84)
        (29, 87)
        (30, 90)
        (31, 93)
        (32, 96)
        (33, 100)
        (34, 103)
        (35, 106)
        (36, 110)
        (37, 113)
        (38, 117)
        (39, 120)
        (40, 124)
        (41, 127)
        (42, 131)
        (43, 135)
        (44, 139)
        (45, 142)
        (46, 146)
        (47, 165)
        (48, 191)
        (49, 218)
        (50, 245)
        (51, 311)
        (52, 812)
    };
\addlegendentry{\aspmcr 500s}

\end{axis}
\end{tikzpicture}
}
\caption{\textit{reachGrid}.}
\label{subfig:reach_grid}
\end{subfigure}%
\hfill
\begin{subfigure}{0.48\textwidth}
\resizebox{\resizeGraphFactor\textwidth}{!}{%
\begin{tikzpicture}
\begin{axis}[
    xlabel={Number of Solved Instances},
    ylabel={Execution Time (s)},
    legend pos=north west,
    legend style={fill=white, fill opacity=0.6},
    grid style=dashed,
    legend cell align={left},every axis plot/.append style={thick}
    ]

\addplot[color = original_plot] 
    coordinates {
        (0, 0)
        (1, 2)
        (2, 5)
        (3, 8)
        (4, 11)
        (5, 14)
        (6, 17)
        (7, 20)
        (8, 23)
        (9, 26)
        (10, 30)
        (11, 77)
        (12, 126)
        (13, 176)
        (14, 228)
        (15, 283)
        (16, 339)
        (17, 398)
        (18, 460)
        (19, 534)
        (20, 5504)
    };
\addlegendentry{aspmc 100s}
\addplot[color = original_plot, dashed] 
    coordinates {
        (0, 0)
        (1, 3)
        (2, 7)
        (3, 11)
        (4, 15)
        (5, 19)
        (6, 23)
        (7, 28)
        (8, 32)
        (9, 37)
        (10, 41)
        (11, 88)
        (12, 138)
        (13, 187)
        (14, 240)
        (15, 295)
        (16, 351)
        (17, 411)
        (18, 473)
        (19, 546)
        (20, 658)
        (21, 5504)
    };
\addlegendentry{aspmc 300s}
\addplot[color = original_plot, dotted] 
    coordinates {
        (0, 0)
        (1, 3)
        (2, 7)
        (3, 11)
        (4, 15)
        (5, 19)
        (6, 24)
        (7, 28)
        (8, 32)
        (9, 37)
        (10, 42)
        (11, 49)
        (12, 58)
        (13, 67)
        (14, 76)
        (15, 85)
        (16, 95)
        (17, 105)
        (18, 116)
        (19, 127)
        (20, 144)
        (21, 5504)
    };
\addlegendentry{aspmc 500s}
\addplot[color = tabled_plot] 
    coordinates {
        (0, 0)
        (1, 2)
        (2, 5)
        (3, 8)
        (4, 11)
        (5, 14)
        (6, 17)
        (7, 20)
        (8, 23)
        (9, 26)
        (10, 29)
        (11, 32)
        (12, 35)
        (13, 38)
        (14, 41)
        (15, 44)
        (16, 47)
        (17, 51)
        (18, 54)
        (19, 57)
        (20, 60)
        (21, 63)
        (22, 66)
        (23, 69)
        (24, 72)
        (25, 75)
        (26, 78)
        (27, 81)
        (28, 84)
        (29, 87)
        (30, 90)
        (31, 94)
        (32, 97)
        (33, 100)
        (34, 103)
        (35, 106)
        (36, 109)
        (37, 112)
        (38, 116)
        (39, 119)
        (40, 122)
        (41, 125)
        (42, 128)
        (43, 132)
        (44, 135)
        (45, 138)
        (46, 141)
        (47, 145)
        (48, 148)
        (49, 151)
        (50, 154)
        (51, 158)
        (52, 161)
        (53, 164)
        (54, 168)
        (55, 171)
        (56, 174)
        (57, 178)
        (58, 181)
        (59, 185)
        (60, 188)
        (61, 192)
        (62, 195)
        (63, 199)
        (64, 202)
        (65, 206)
        (66, 209)
        (67, 213)
        (68, 217)
        (69, 221)
        (70, 225)
        (71, 228)
        (72, 232)
        (73, 236)
        (74, 240)
        (75, 244)
        (76, 248)
        (77, 252)
        (78, 256)
        (79, 260)
        (80, 265)
        (81, 269)
        (82, 273)
        (83, 277)
        (84, 281)
        (85, 286)
        (86, 290)
        (87, 294)
        (88, 298)
        (89, 303)
        (90, 307)
        (91, 311)
        (92, 316)
        (93, 321)
        (94, 326)
        (95, 330)
        (96, 335)
        (97, 340)
        (98, 345)
        (99, 350)
        (100, 355)
        (101, 360)
        (102, 365)
        (103, 370)
        (104, 375)
        (105, 380)
        (106, 386)
        (107, 392)
        (108, 398)
        (109, 405)
        (110, 412)
        (111, 419)
        (112, 427)
        (113, 435)
        (114, 443)
        (115, 452)
        (116, 460)
        (117, 469)
        (118, 479)
        (119, 489)
        (120, 499)
        (121, 509)
        (122, 520)
        (123, 531)
        (124, 542)
        (125, 554)
        (126, 568)
        (127, 583)
        (128, 599)
        (129, 615)
        (130, 639)
        (131, 667)
        (132, 697)
        (133, 727)
        (134, 757)
        (135, 789)
        (136, 821)
        (137, 858)
        (138, 901)
        (139, 945)
        (140, 990)
        (141, 1036)
        (142, 1087)
        (143, 1149)
        (144, 1215)
        (145, 1283)
        (146, 1354)
        (147, 5504)
    };
\addlegendentry{\aspmcr 100s}
\addplot[color = tabled_plot, dashed] 
    coordinates {
        (0, 0)
        (1, 2)
        (2, 5)
        (3, 8)
        (4, 11)
        (5, 14)
        (6, 17)
        (7, 20)
        (8, 23)
        (9, 26)
        (10, 29)
        (11, 32)
        (12, 35)
        (13, 38)
        (14, 41)
        (15, 44)
        (16, 47)
        (17, 50)
        (18, 54)
        (19, 57)
        (20, 60)
        (21, 63)
        (22, 66)
        (23, 69)
        (24, 72)
        (25, 75)
        (26, 78)
        (27, 81)
        (28, 84)
        (29, 87)
        (30, 91)
        (31, 94)
        (32, 97)
        (33, 100)
        (34, 103)
        (35, 106)
        (36, 109)
        (37, 113)
        (38, 116)
        (39, 119)
        (40, 122)
        (41, 125)
        (42, 129)
        (43, 132)
        (44, 135)
        (45, 138)
        (46, 142)
        (47, 145)
        (48, 148)
        (49, 151)
        (50, 155)
        (51, 158)
        (52, 161)
        (53, 165)
        (54, 168)
        (55, 172)
        (56, 175)
        (57, 179)
        (58, 182)
        (59, 186)
        (60, 189)
        (61, 193)
        (62, 197)
        (63, 200)
        (64, 204)
        (65, 208)
        (66, 212)
        (67, 216)
        (68, 220)
        (69, 223)
        (70, 227)
        (71, 231)
        (72, 235)
        (73, 239)
        (74, 243)
        (75, 247)
        (76, 252)
        (77, 256)
        (78, 260)
        (79, 264)
        (80, 268)
        (81, 272)
        (82, 276)
        (83, 280)
        (84, 285)
        (85, 289)
        (86, 293)
        (87, 297)
        (88, 302)
        (89, 306)
        (90, 311)
        (91, 316)
        (92, 320)
        (93, 325)
        (94, 330)
        (95, 335)
        (96, 339)
        (97, 344)
        (98, 349)
        (99, 354)
        (100, 359)
        (101, 364)
        (102, 370)
        (103, 375)
        (104, 381)
        (105, 388)
        (106, 395)
        (107, 402)
        (108, 409)
        (109, 416)
        (110, 423)
        (111, 431)
        (112, 440)
        (113, 448)
        (114, 457)
        (115, 465)
        (116, 475)
        (117, 484)
        (118, 494)
        (119, 505)
        (120, 515)
        (121, 526)
        (122, 536)
        (123, 547)
        (124, 559)
        (125, 572)
        (126, 586)
        (127, 602)
        (128, 619)
        (129, 636)
        (130, 658)
        (131, 683)
        (132, 711)
        (133, 740)
        (134, 769)
        (135, 799)
        (136, 830)
        (137, 862)
        (138, 897)
        (139, 941)
        (140, 986)
        (141, 1032)
        (142, 1078)
        (143, 1140)
        (144, 1202)
        (145, 1264)
        (146, 1328)
        (147, 1394)
        (148, 1492)
        (149, 1621)
        (150, 1751)
        (151, 1885)
        (152, 2032)
        (153, 2179)
        (154, 2355)
        (155, 2569)
        (156, 2816)
        (157, 3094)
        (158, 5504)
    };
\addlegendentry{\aspmcr 300s}
\addplot[color = tabled_plot, dotted] 
    coordinates {
        (0, 0)
        (1, 2)
        (2, 5)
        (3, 8)
        (4, 11)
        (5, 14)
        (6, 17)
        (7, 20)
        (8, 23)
        (9, 26)
        (10, 29)
        (11, 32)
        (12, 35)
        (13, 38)
        (14, 41)
        (15, 44)
        (16, 47)
        (17, 50)
        (18, 53)
        (19, 56)
        (20, 59)
        (21, 63)
        (22, 66)
        (23, 69)
        (24, 72)
        (25, 75)
        (26, 78)
        (27, 81)
        (28, 84)
        (29, 87)
        (30, 90)
        (31, 93)
        (32, 96)
        (33, 99)
        (34, 103)
        (35, 106)
        (36, 109)
        (37, 112)
        (38, 115)
        (39, 118)
        (40, 121)
        (41, 125)
        (42, 128)
        (43, 131)
        (44, 134)
        (45, 138)
        (46, 141)
        (47, 144)
        (48, 148)
        (49, 151)
        (50, 154)
        (51, 157)
        (52, 161)
        (53, 164)
        (54, 168)
        (55, 171)
        (56, 175)
        (57, 178)
        (58, 182)
        (59, 185)
        (60, 189)
        (61, 192)
        (62, 196)
        (63, 199)
        (64, 203)
        (65, 207)
        (66, 211)
        (67, 215)
        (68, 219)
        (69, 223)
        (70, 227)
        (71, 231)
        (72, 235)
        (73, 239)
        (74, 243)
        (75, 247)
        (76, 251)
        (77, 255)
        (78, 259)
        (79, 263)
        (80, 267)
        (81, 271)
        (82, 275)
        (83, 279)
        (84, 284)
        (85, 288)
        (86, 292)
        (87, 296)
        (88, 301)
        (89, 305)
        (90, 310)
        (91, 314)
        (92, 319)
        (93, 324)
        (94, 328)
        (95, 333)
        (96, 338)
        (97, 343)
        (98, 348)
        (99, 353)
        (100, 358)
        (101, 363)
        (102, 368)
        (103, 374)
        (104, 379)
        (105, 386)
        (106, 393)
        (107, 400)
        (108, 407)
        (109, 415)
        (110, 422)
        (111, 430)
        (112, 438)
        (113, 447)
        (114, 455)
        (115, 463)
        (116, 473)
        (117, 482)
        (118, 492)
        (119, 502)
        (120, 513)
        (121, 523)
        (122, 533)
        (123, 544)
        (124, 556)
        (125, 569)
        (126, 582)
        (127, 598)
        (128, 615)
        (129, 632)
        (130, 657)
        (131, 685)
        (132, 713)
        (133, 741)
        (134, 771)
        (135, 802)
        (136, 832)
        (137, 864)
        (138, 900)
        (139, 942)
        (140, 986)
        (141, 1032)
        (142, 1081)
        (143, 1141)
        (144, 1201)
        (145, 1263)
        (146, 1325)
        (147, 1387)
        (148, 1481)
        (149, 1610)
        (150, 1742)
        (151, 1880)
        (152, 2020)
        (153, 2162)
        (154, 2342)
        (155, 2559)
        (156, 2807)
        (157, 3065)
        (158, 3396)
        (159, 3770)
        (160, 4164)
        (161, 4607)
        (162, 5054)
        (163, 5504)
    };
\addlegendentry{\aspmcr 500s}

\end{axis}
\end{tikzpicture}
}
\caption{\textit{reachBA}.}
\label{subfig:reach_ba}
\end{subfigure}
\caption{Cactus plot for aspmc and \aspmcr with 100, 300, and 500 seconds of time limit on the \textit{reachGrid} and \textit{reachBA} datasets.}
\label{fig:cactus_reach}
\end{figure}
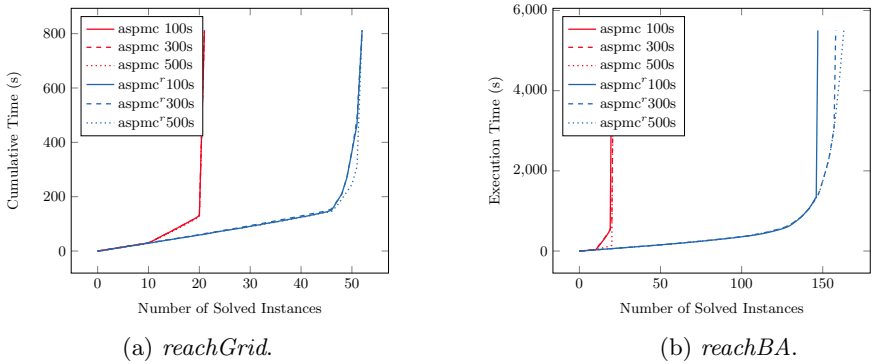

In the following, aspmc denotes the results obtained by applying aspmc directly on the considered instance while \aspmcr\ denotes the results obtained by first computing the residual program and then passing it to aspmc.
For all the experiments the extraction of the residual program was done using the predicate $call\_residual\_program/2$ available in SWI.
The extraction takes less than one second, so we decided to report only the total execution times, without indicating the two components for \aspmcr.
This is also the motivation behind the decision of testing only one Prolog system, namely SWI.
We could have also used  XSB but the results would not have been  much different given the almost instantaneous extraction of the residual program.
Given the probabilistic nature of the generation of Barabasi-Albert graphs and of the query for grid graphs, the results are averaged over 10 runs.
For \textit{reachBA} and \textit{smokersBA}, the query is the same in every run but the structure of the graph changes in every run (i.e., each instance has a different graph structure).
For \textit{reachGrid} and \textit{smokersGrid}, the grid graph is the same in each of the 10 runs but the query changes in each attempt.
Figure~\ref{fig:cactus_smk} and~\ref{fig:cactus_reach} show the cactus plot for the four datasets.
For \textit{smokersGrid} (Figure~\ref{subfig:cactus_smk_grid}), bare aspmc cannot solve more than 20 instances while \aspmcr\ can solve up to 60.
Similar considerations hold for \textit{smokersBA} (Figure~\ref{subfig:cactus_smk_ba}), where \aspmcr\ can go to over 100 instances solved while aspmc stops at 10.
For both datasets, the timeout seems to not influence too much the number of solvable instances, since most of the curves almost coincide.
Analogous considerations apply for \textit{reachGrid} (Figure~\ref{subfig:reach_grid}) and \textit{reachBA} (Figure~\ref{subfig:reach_ba}).
The improvement provided by the residual program extraction can be further assessed from Table~\ref{tab:tab_ba} and Table~\ref{tab:tab_grid}, reporting the number of solved instances and the average number of bags, the average threewidth, and the average number of vertices obtained from the tree decomposition performed by aspmc, with and without residual program extraction.
Note that the result of the treewidth decomposition is available even if the instance cannot be solved within the time limit (i.e., the averages are always over the 10 runs).
For example, for size 80 of the \textit{reachBA} dataset, the number of bags goes from 1956 to 47, after the residual program extraction. 
Overall, the residual program extraction has a huge impact on the simplification of the program, both in terms of more compact representation and in terms of execution time.

\section{Related Works}
\label{sec:related}
The residual program extraction is at the heart of PITA~\citep{riguzzi2011pita} and ProbLog2~\citep{dries2015problog2}, the first adopting Prolog SLG resolution to caching the part of the programs that has already been analyzed.
There are other semantics to represent uncertainty with an answer set program such as LPMLN~\citep{DBLP:conf/aaai/LeeY17}, P-log~\citep{DBLP:journals/tplp/BaralGR09} or smProbLog~\citep{totis_de_raedt_kimmig_2023}.
LPMLN allows defining weighted rules and assigns weights to answer sets while P-log adopts probabilistic facts but requires normalization for the computation of the probability.
Furthermore, P-log has an interface built on top of XSB~\citep{plog_xasp} leveraging its tabling mechanisms to speed up inference.
The relation between the two has been studied in detail~\citep{DBLP:conf/ijcai/BalaiG16,DBLP:conf/aaai/LeeY17}.
Another possibility to associate weights to rules is via weak constraints, available in all ASP solvers, that however cannot be directly interpreted as probabilities.
smProbLog is the semantics closest to the CS: both support probabilistic facts added on top of an ASP.
The probability of a stable model in smProbLog is the probability of its corresponding world $w$ divided by the number of answer sets of $w$.
The CS has also been extended by~\cite{DBLP:conf/kr/RochaC22} to also handle worlds without answer sets, but it requires three truth values (true, false, and undefined).
The residual program extraction may help to speed up inference also in these alternative semantics: exploring this is an interesting future work.
We consider the SLG resolution implemented in SWI Prolog.
However, as already discussed in Section~\ref{subsec:tabling}, it was initially proposed and implemented in the XSB system~\citep{DBLP:journals/tplp/SwiftW12}.
Our approach is general and can be built on top of any Prolog system that supports SLG resolution.

The problem of grounding in ASP has also been addressed by the s(ASP)~\citep{DBLP:journals/corr/abs-1709-00501} and s(CASP)~\cite{arias_carro_salazar_marple_gupta_2018} systems, which are top-down goal-driven ASP interpreters (the latter also allowing constraints).
The result of a query in these systems is a subset of the stable models of the whole program containing only the atoms needed to prove the query.
Furthermore, the evaluation of a query does not need to ground the whole program. 
s(ASP) is based on several techniques combined together, such as coinductive SLD resolution to handle cycles through negation and constructive negation based on dual rules 
to identify why a particular query fails.
DLV~\citep{DBLP:journals/tocl/LeonePFEGPS06} also has also a query mode.
However, none of these systems target probabilistic inference.
Furthermore, our approach does not aim to replace the ASP solver, rather to reduce the size of the program that should be grounded.
Another possibility to extract the residual program is by analyzing the dependency graph.
However, to do so, it is often needed to ground the whole program.
With our approach based on SLG resolution, only the relevant part of the program is grounded. 

\section{Conclusions}
\label{sec:conclusions}
In this paper we proposed to speed up inference in PASP via extraction of the \textit{residual} program.
The residual program represents the part of the program that is needed to compute the probability of a query and it is often smaller than the original program.
This allows a reasoner to ground a smaller portion of the program to compute the probability of a query, reducing the execution time.
We extract the residual program by applying SLG resolution and tabling.
Empirical results on graph datasets shows that i) the time spent to extract the residual program is negligible w.r.t. the inference time and ii) querying the residual program is much faster than querying the original program.

\section*{Acknowledgements} 
This work has been partially supported by Spoke 1 ``FutureHPC \& BigData'' of the Italian Research Center on High-Performance Computing, Big Data and Quantum Computing (ICSC) funded by MUR Missione 4 - Next Generation EU (NGEU) and by Partenariato Esteso PE00000013 - ``FAIR - Future Artificial Intelligence Research'' - Spoke 8 ``Pervasive AI'', funded by MUR through PNRR - M4C2 - Investimento 1.3 (Decreto Direttoriale MUR n. 341 of 15th March 2022) under the Next Generation EU (NGEU).
Both authors are members of the Gruppo Nazionale Calcolo Scientifico -- Istituto Nazionale di Alta Matematica (GNCS-INdAM).

\subsection*{Competing Interests}
The authors declare none.

\bibliographystyle{apalike}
\bibliography{biblio}

\begin{thebibliography}{}

\bibitem[Anh et~al., 2008]{plog_xasp}
Anh, H.~T., Kencana~Ramli, C. D.~P., and Dam{\'a}sio, C.~V. (2008).
\newblock An implementation of extended {P}-log using {XASP}.
\newblock In Garcia de~la Banda, M. and Pontelli, E., editors, {\em Logic Programming}, pages 739--743, Berlin, Heidelberg. Springer Berlin Heidelberg.

\bibitem[Arias et~al., 2018]{arias_carro_salazar_marple_gupta_2018}
Arias, J., Carro, M., Salazar, E., Marple, K., and Gupta, G. (2018).
\newblock Constraint answer set programming without grounding.
\newblock {\em Theory and Practice of Logic Programming}, 18(3-4):337--354.

\bibitem[Azzolini et~al., 2022]{AzzBellRig2022PASTA}
Azzolini, D., Bellodi, E., and Riguzzi, F. (2022).
\newblock Statistical statements in probabilistic logic programming.
\newblock In Gottlob, G., Inclezan, D., and Maratea, M., editors, {\em Logic Programming and Nonmonotonic Reasoning}, pages 43--55, Cham. Springer International Publishing.

\bibitem[Azzolini and Riguzzi, 2023]{AzzRig2023-AIXIA-IC}
Azzolini, D. and Riguzzi, F. (2023).
\newblock Inference in probabilistic answer set programming under the credal semantics.
\newblock In Basili, R., Lembo, D., Limongelli, C., and Orlandini, A., editors, {\em {AIxIA} 2023 - Advances in Artificial Intelligence}, volume 14318 of {\em Lecture Notes in Artificial Intelligence}, pages 367--380, Heidelberg, Germany. Springer.

\bibitem[Balai and Gelfond, 2016]{DBLP:conf/ijcai/BalaiG16}
Balai, E. and Gelfond, M. (2016).
\newblock On the relationship between {P-log} and {LP}\({}^{{\mbox{{MLN}}}}\).
\newblock In Kambhampati, S., editor, {\em Proceedings of the Twenty-Fifth International Joint Conference on Artificial Intelligence, {IJCAI} 2016, New York, NY, USA, 9-15 July 2016}, pages 915--921. {IJCAI/AAAI} Press.

\bibitem[Baral et~al., 2009]{DBLP:journals/tplp/BaralGR09}
Baral, C., Gelfond, M., and Rushton, N. (2009).
\newblock Probabilistic reasoning with answer sets.
\newblock {\em Theory and Practice of Logic Programming}, 9(1):57--144.

\bibitem[Bodlaender et~al., 1993]{bodlaender1993tourist}
Bodlaender, H.~L. et~al. (1993).
\newblock A tourist guide through treewidth.
\newblock {\em Acta Cybernetica}, 11(1-2):1--21.

\bibitem[Brewka et~al., 2011]{brewka2011asp}
Brewka, G., Eiter, T., and Truszczy\'{n}ski, M. (2011).
\newblock Answer set programming at a glance.
\newblock {\em Communications of the ACM}, 54(12):92--103.

\bibitem[Chen and Warren, 1996]{DBLP:journals/jacm/ChenW96}
Chen, W. and Warren, D.~S. (1996).
\newblock Tabled evaluation with delaying for general logic programs.
\newblock {\em Journal of the ACM}, 43(1):20--74.

\bibitem[Cozman and Mau{\'{a}}, 2020]{cozman2020pasp}
Cozman, F.~G. and Mau{\'{a}}, D.~D. (2020).
\newblock The joy of probabilistic answer set programming: Semantics, complexity, expressivity, inference.
\newblock {\em International Journal of Approximate Reasoning}, 125:218--239.

\bibitem[Darwiche and Marquis, 2002]{DBLP:journals/jair/DarwicheM02}
Darwiche, A. and Marquis, P. (2002).
\newblock A knowledge compilation map.
\newblock {\em Journal of Artificial Intelligence Research}, 17:229--264.

\bibitem[{De Raedt} et~al., 2007]{DBLP:conf/ijcai/RaedtKT07}
{De Raedt}, L., Kimmig, A., and Toivonen, H. (2007).
\newblock {ProbLog}: A probabilistic {Prolog} and its application in link discovery.
\newblock In Veloso, M.~M., editor, {\em 20th International Joint Conference on Artificial Intelligence (IJCAI 2007)}, volume~7, pages 2462--2467. AAAI Press.

\bibitem[Dix, 1995]{10.5555/2383266.2383269}
Dix, J. (1995).
\newblock A classification theory of semantics of normal logic programs: I. strong properties.
\newblock {\em Fundamenta Informaticae}, 22(3):227--255.

\bibitem[Dries et~al., 2015]{dries2015problog2}
Dries, A., Kimmig, A., Meert, W., Renkens, J., Van~den Broeck, G., Vlasselaer, J., and De~Raedt, L. (2015).
\newblock {ProbLog2}: Probabilistic logic programming.
\newblock In {\em European Conference on Machine Learning and Principles and Practice of Knowledge Discovery in Databases (ECMLPKDD 2015)}, volume 9286 of {\em Lecture Notes in Computer Science}, pages 312--315. Springer.

\bibitem[Eiter et~al., 2021]{DBLP:conf/kr/EiterHK21}
Eiter, T., Hecher, M., and Kiesel, R. (2021).
\newblock Treewidth-aware cycle breaking for algebraic answer set counting.
\newblock In Bienvenu, M., Lakemeyer, G., and Erdem, E., editors, {\em Proceedings of the 18th International Conference on Principles of Knowledge Representation and Reasoning, {KR} 2021}, pages 269--279.

\bibitem[Eiter et~al., 2024]{EITER2024104109}
Eiter, T., Hecher, M., and Kiesel, R. (2024).
\newblock aspmc: New frontiers of algebraic answer set counting.
\newblock {\em Artificial Intelligence}, 330:104109.

\bibitem[Gebser et~al., 2009]{gebser2009projective}
Gebser, M., Kaufmann, B., and Schaub, T. (2009).
\newblock Solution enumeration for projected boolean search problems.
\newblock In van Hoeve, W.-J. and Hooker, J.~N., editors, {\em Integration of AI and OR Techniques in Constraint Programming for Combinatorial Optimization Problems}, pages 71--86, Berlin, Heidelberg. Springer Berlin Heidelberg.

\bibitem[Gelfond and Lifschitz, 1988]{gelfond1988stable}
Gelfond, M. and Lifschitz, V. (1988).
\newblock The stable model semantics for logic programming.
\newblock In {\em 5th International Conference and Symposium on Logic Programming (ICLP/SLP 1988)}, volume~88, pages 1070--1080, USA. MIT Press.

\bibitem[Hagberg et~al., 2008]{hagberg2008networkx}
Hagberg, A.~A., Schult, D.~A., and Swart, P.~J. (2008).
\newblock Exploring network structure, dynamics, and function using {NetworkX}.
\newblock In Varoquaux, G., Vaught, T., and Millman, J., editors, {\em Proceedings of the 7th Python in Science Conference}, pages 11--15, Pasadena, CA USA.

\bibitem[Kiesel et~al., 2022]{DBLP:journals/tplp/KieselTK22}
Kiesel, R., Totis, P., and Kimmig, A. (2022).
\newblock Efficient knowledge compilation beyond weighted model counting.
\newblock {\em Theory and Practice of Logic Programming}, 22(4):505--522.

\bibitem[Kimmig et~al., 2017]{10.1016/j.jal.2016.11.031}
Kimmig, A., Van~den Broeck, G., and De~Raedt, L. (2017).
\newblock Algebraic model counting.
\newblock {\em Journal of Applied Logic}, 22(C):46--62.

\bibitem[Lee and Yang, 2017]{DBLP:conf/aaai/LeeY17}
Lee, J. and Yang, Z. (2017).
\newblock {LPMLN}, weak constraints, and {P-log}.
\newblock In Singh, S. and Markovitch, S., editors, {\em Proceedings of the Thirty-First {AAAI} Conference on Artificial Intelligence, February 4-9, 2017, San Francisco, California, {USA}}, pages 1170--1177. {AAAI} Press.

\bibitem[Leone et~al., 2006]{DBLP:journals/tocl/LeonePFEGPS06}
Leone, N., Pfeifer, G., Faber, W., Eiter, T., Gottlob, G., Perri, S., and Scarcello, F. (2006).
\newblock The {DLV} system for knowledge representation and reasoning.
\newblock {\em ACM Transactions on Computational Logic}, 7(3):499--562.

\bibitem[Marple and Gupta, 2014]{DBLP:journals/tplp/MarpleG14}
Marple, K. and Gupta, G. (2014).
\newblock Dynamic consistency checking in goal-directed answer set programming.
\newblock {\em Theory and Practice of Logic Programming}, 14(4-5):415--427.

\bibitem[Marple et~al., 2017]{DBLP:journals/corr/abs-1709-00501}
Marple, K., Salazar, E., and Gupta, G. (2017).
\newblock Computing stable models of normal logic programs without grounding.
\newblock {\em CoRR}, abs/1709.00501.

\bibitem[Przymusinski, 1989]{Przymusinski89}
Przymusinski, T.~C. (1989).
\newblock Every logic program has a natural stratification and an iterated least fixed point model.
\newblock In {\em Proceedings of the Eighth ACM SIGACT-SIGMOD-SIGART Symposium on Principles of Database Systems}, PODS '89, pages 11--21, New York, NY, USA. Association for Computing Machinery.

\bibitem[Raedt et~al., 2016]{raedt2016statistical}
Raedt, L.~D., Kersting, K., Natarajan, S., and Poole, D. (2016).
\newblock Statistical relational artificial intelligence: Logic, probability, and computation.
\newblock {\em Synthesis Lectures on Artificial Intelligence and Machine Learning}, 10(2):1--189.

\bibitem[Riguzzi, 2022]{Rig23-BKaddress}
Riguzzi, F. (2022).
\newblock {\em Foundations of Probabilistic Logic Programming Languages, Semantics, Inference and Learning, Second Edition}.
\newblock River Publishers, Gistrup, Denmark.

\bibitem[Riguzzi and Swift, 2011]{riguzzi2011pita}
Riguzzi, F. and Swift, T. (2011).
\newblock The {PITA} system: Tabling and answer subsumption for reasoning under uncertainty.
\newblock {\em Theory and Practice of Logic Programming}, 11(4-5):433--449.

\bibitem[Rocha and Gagliardi~Cozman, 2022]{DBLP:conf/kr/RochaC22}
Rocha, V. H.~N. and Gagliardi~Cozman, F. (2022).
\newblock A credal least undefined stable semantics for probabilistic logic programs and probabilistic argumentation.
\newblock In Kern{-}Isberner, G., Lakemeyer, G., and Meyer, T., editors, {\em Proceedings of the 19th International Conference on Principles of Knowledge Representation and Reasoning, {KR} 2022}, pages 309--319.

\bibitem[Swift, 1999]{DBLP:conf/epia/Swift99}
Swift, T. (1999).
\newblock A new formulation of tabled resolution with delay.
\newblock In Barahona, P. and Alferes, J.~J., editors, {\em Progress in Artificial Intelligence, 9th Portuguese Conference on Artificial Intelligence, EPIA '99, \'{E}vora, Portugal, September 21-24, 1999, Proceedings}, volume 1695 of {\em Lecture Notes in Computer Science}, pages 163--177, Berlin. Springer.

\bibitem[Swift and Warren, 2012]{DBLP:journals/tplp/SwiftW12}
Swift, T. and Warren, D.~S. (2012).
\newblock {XSB}: Extending {Prolog} with tabled logic programming.
\newblock {\em Theory and Practice of Logic Programming}, 12(1-2):157--187.

\bibitem[Totis et~al., 2023]{totis_de_raedt_kimmig_2023}
Totis, P., De~Raedt, L., and Kimmig, A. (2023).
\newblock {smProbLog}: Stable model semantics in {ProbLog} for probabilistic argumentation.
\newblock {\em Theory and Practice of Logic Programming}, pages 1--50.

\bibitem[Van~Gelder et~al., 1991]{well-founded}
Van~Gelder, A., Ross, K.~A., and Schlipf, J.~S. (1991).
\newblock The well-founded semantics for general logic programs.
\newblock {\em Journal of the ACM}, 38(3):620--650.

\bibitem[Wielemaker et~al., 2012]{wielemaker2012swi}
Wielemaker, J., Schrijvers, T., Triska, M., and Lager, T. (2012).
\newblock {SWI-Prolog}.
\newblock {\em Theory and Practice of Logic Programming}, 12(1-2):67--96.

\end{thebibliography}

\end{document}